\documentclass[runningheads]{llncs}

\usepackage[utf8]{inputenc}
\usepackage[T1]{fontenc}
\usepackage{microtype}
\usepackage{fullpage}
\usepackage{amsmath, amssymb}
\usepackage{mathtools}
\usepackage{booktabs}
\usepackage[algoruled,linesnumbered]{algorithm2e}
\usepackage[inline]{enumitem}
\usepackage{tikz}
\usepackage[hidelinks]{hyperref} 

\usetikzlibrary{automata, arrows, petri, positioning, shapes}
\tikzset{auto, >=stealth}
\tikzset{every edge/.append style={shorten >=1pt}}
\tikzstyle{decision} = [diamond, draw, fill=blue!20, text width=4.5em, text badly centered, node distance=3cm, inner sep=0pt]
\tikzstyle{block} = [rectangle, draw, fill=blue!20, text width=6em, text centered, rounded corners, minimum height=4em]
\tikzstyle{greenblock} = [rectangle, draw, fill=green!20, text width=6em, text centered, rounded corners, minimum height=4em]
\tikzstyle{redblock} = [rectangle, draw, fill=red!20, text width=6em, text centered, rounded corners, minimum height=4em]
\tikzstyle{varblock} = [rectangle, draw, fill=yellow!20, text width=6em, text centered, rounded corners, minimum height=4em]
\tikzstyle{bigblock} = [rectangle, draw, fill=gray!20, text width=6.6em, text centered, rounded corners, minimum height=13.2em]
\tikzstyle{line} = [draw, -latex']
\tikzstyle{cloud} = [draw, ellipse,fill=red!20, node distance=3cm, minimum height=2em]

\makeatletter
\renewcommand{\the@inst}{\alph{@inst}}
\makeatother

\newcommand{\RNN}{R}

\newcommand{\transrnn}{f}
\newcommand{\accrnn}{g}
\newcommand{\initrnn}{h_0}

\newcommand{\Prob}{P}

\newcommand{\tprob}{p}
\newcommand{\ntprob}{p}

\newcommand{\width}{\epsilon}
\newcommand{\confidence}{\gamma}
\newcommand{\letterprobp}[1]{p_{#1}}

\newcommand{\colorcounter}[1]{\textcolor{red}{#1}}
\newcommand{\colorsat}[1]{\textcolor{teal}{#1}}

\newcommand{\ext}[1]{\smash{\hat{#1}}}
\newcommand{\sdim}{k}

\newcommand{\R}{\mathbb{R}}
\newcommand{\N}{\mathbb{N}}
\newcommand{\A}{A}
\newcommand{\Hyp}{\mathcal{H}}
\newcommand{\PropAut}{\A}
\newcommand{\Lstar}{\textup{L}^\ast}
\newcommand{\emptyword}{\lambda}
\renewcommand{\epsilon}{\varepsilon}

\newcommand{\temporalnet}{G}
\newcommand{\vertices}{V}
\newcommand{\edges}{E}

\DeclarePairedDelimiter{\floor}{\lfloor}{\rfloor}

\newcommand\subsetsim{\mathrel{%
  \ooalign{\raise0.2ex\hbox{$\subset$}\cr\hidewidth\raise-0.8ex\hbox{\scalebox{0.9}{$\sim$}}\hidewidth\cr}}}

\newcommand{\randA}{\A_{\mathsf{rand}}}
\newcommand{\specA}{\A}

\title{Property-Directed Verification of Recurrent Neural Networks\thanks{The first three authors contributed equally, the remaining authors are ordered alphabetically}}

\author{
	Igor Khmelnitsky\inst{a,b} \and
	Daniel Neider\inst{c} \and
	Rajarshi Roy\inst{c} \and
	Benoît Barbot\inst{d} \and
	Benedikt Bollig\inst{a} \and
	Alain Finkel\inst{a,g} \and\\
	Serge Haddad\inst{a,b} \and
	Martin Leucker \inst{e} \and
	Lina Ye\inst{a,b,f}
}
\authorrunning{I. Khmelnitsky et al.}

\institute{
	LSV, CNRS, ENS Paris-Saclay, Universit{\'e} Paris-Saclay, France \and
	Inria, France \and
	Max Planck Institute for Software Systems, Kaiserslautern, Germany \and
	Universit{\'e} Paris-Est Cr{\'e}teil, France \and
	Institute for Software Engineering and Programming Languages, Universit{\"a}t zu L{\"u}beck, Germany \and
	CentraleSup{\'e}lec, Universit{\'e} Paris-Saclay, France \and
	Institut Universitaire de France, France
}

\begin{document}

\maketitle
\setcounter{footnote}{0} 

\begin{abstract}
This paper presents a property-directed approach to verifying recurrent neural networks (RNNs). 
To this end, we learn a deterministic finite automaton as a \emph{surrogate model} from a given RNN using active automata learning.
This model may then be analyzed using \emph{model checking} as verification technique. 
The term \emph{property-directed} reflects the idea that our procedure is guided and controlled by the given property rather than performing the two steps separately. 
We show that this not only allows us to discover \emph{small} counterexamples fast, but also to generalize them by pumping towards faulty flows hinting at the underlying error in the RNN.
\end{abstract}


\section{Introduction}
\label{sec:intro}

Recurrent neural networks (RNNs) are a state-of-the-art tool to
represent and learn sequence-based models.
They have applications in time-series prediction, sentiment analysis,
and many more. In particular,
they are increasingly used in safety-critical applications
and act, for example, as controllers in cyber-physical systems \cite{AkintundeKLP19}.
Thus, there is a growing need for formal verification.
However, research in this domain is only at the beginning.
While formal-methods based techniques such as \emph{model checking} \cite{BK2008}
have been successfully used in practice and reached a certain
level of industrial acceptance, a transfer to machine-learning
algorithms has yet to take place.

A recent research stream aims at extracting, from RNNs,
state-based formalisms such as finite automata.
Finite automata turned out to be useful for understanding
and analyzing all kind of systems using testing or model checking. 
In the field of formal verification, it has proven to be beneficial to run
the extraction and verification process simultaneously.
Moreover, the state space of RNNs tends to be prohibitively large, or even infinite,
and so do incremental abstractions thereof.
Motivated by these facts, we propose an intertwined approach to verifying RNNs,
where, in an incremental
fashion, grammatical inference and model checking go hand-in-hand.
Our approach is inspired by black-box checking \cite{PeledVY02} where one \emph{exploits}
the property to be verified \emph{during} the verification
process. Our procedure can be used to find misclassified
(positive and negative) examples or to
verify a system that the given RNN controls.

\paragraph{Property-directed verification.}
Let us give a glimpse of our method.
We consider an RNN $\RNN$ as a binary classifier of finite sequences over
a finite alphabet $\Sigma$. In other words, $R$ represents the set of strings
that are classified as positive. We denote this set by $L(\RNN)$
and call it the \emph{language} of $\RNN$.
Note that $L(\RNN) \subseteq \Sigma^\ast$.
We would like to know whether $R$ is compatible with a given
specification $\A$, written $\RNN \models \A$.
Here, we assume that $\A$ is given as a (deterministic) finite automaton.
Finite automata are algorithmically feasible, albeit
having a reasonable expressive power: many abstract specification languages
such as temporal logics or regular expressions can be compiled into finite automata \cite{GiacomoV15}.

But what does $\RNN \models \A$ actually mean?
In fact, there are various options. If $\A$ provides a complete characterization
of the sequences that are to be classified as positive, then $\models$
refers to language equivalence, i.e., $L(\RNN) = L(A)$.
Note that this would imply that $L(\RNN)$ is supposed to be a regular language, which may
rarely be the case in practice.
Therefore, we will focus on the more versatile inclusion $L(\RNN) \subseteq L(\A)$.
This amounts to checking that
$\RNN$ does not produce false positives (wrt.\ the property $\A$).
That is, all strings classified as positive by $\RNN$ must be included in the
specification. In general, $\A$ is a regular abstraction of the concept that
$\RNN$ is supposed to represent. It contains words where a
positive classification is still considered acceptable.
In turn, the complement $\Sigma^\ast \setminus L(\A)$
contains the words that $\RNN$ \emph{must classify as negative} at any cost.
As finite automata can be complemented, we can deal with any
of these interpretations.

For instance, assume that $\RNN$ is supposed to recognize valid
XML documents, which, seen as a set of strings, is not a regular language.
We may want to make sure that every opening tag is eventually
followed by a corresponding closing tag\footnote{over a finite predefined
set of tags, which are contained in $\Sigma$}, though the number
of opening and the number of closing tags may differ. As specification, we can then
take an automaton $\A$ accepting the corresponding \emph{regular} set
of strings. For example, $L(\A)$ contains
\begin{center}
\texttt{\small<book><author><author></author></book>}.
\end{center}
However, it does not contain
\begin{center}
\texttt{\small<book><author></author><author></book>}
\end{center}
since the second occurrence of \texttt{\small<author>} is not followed
by some \texttt{\small</author>} anymore.

Symmetrically, our procedure can be used to find
false negative classifications:
If $\A$ represents the words that $\RNN$ \emph{must classify as
positive}, we can run our procedure using the complement of $\A$ as specification
and inverting the outputs of $\RNN$. For example, we may
then choose $\A$ such that $L(\A)$ is the set of XML documents
that $\RNN$ was trained on.

Interpreting $\models$ as inclusion also corresponds to the model-checking problem in formal verification.
In particular, if $\RNN$ acts as a controller of a system and therefore represents
a set of system runs, this would allow us to deduce that the entire system meets the specification.

Now let us briefly explain how we check $\RNN \models \A$, i.e., the inclusion
$L(\RNN) \subseteq L(A)$. A detailed exposition can be found in
Section~\ref{sec:prop-directed}.
Our approach is schematically depicted in Figure~\ref{fig:work-flow}.
At its heart is Angluin's L* learning algorithm \cite{Angluin87},
which tries to find a finite automaton approximating a given
black-box system (in our case, the RNN $\RNN$).
To do so, it produces a sequel of \emph{hypothesis} automata
based on queries to the RNN.
The crux is that every such hypothesis $\Hyp$ can be compared against
the specification $\A$ using classical model-checking algorithms.
Suppose $L(\Hyp) \subseteq L(\A)$ does not hold (right branch in the figure).
If a corresponding counterexample, i.e., a word that is contained in
$L(\Hyp)$ but not in $L(\A)$, is confirmed by the RNN,
then we found a string that was mistakenly classified as positive by the RNN.
Otherwise, the counterexample can be fed back to the algorithm
to refine the hypothesis.
If, on the other hand, we have $L(\Hyp) \subseteq L(\A)$
(left branch in the figure),
then a comparison of $\RNN$ with the hypothesis is necessary.
Again, a corresponding counterexample serves for refinement
of the hypothesis, whereas inclusion $L(\RNN) \subseteq L(\Hyp)$
allows us to deduce the correctness of $\RNN$.
It remains of course to clarify how we test this inclusion,
because $L(\RNN)$ is still an unknown, possibly not regular
language. We rely here on statistical model checking.
Just like in black-box checking, our experimental results suggest
that the process of interweaving automata learning and model checking
is beneficial in the verification of RNNs and offers advantages
over more obvious approaches such as pure statistical model checking.

Note that, though we only cover the case of binary classifiers,
our framework is in principle applicable to multiple labels
using one-vs-all classification.

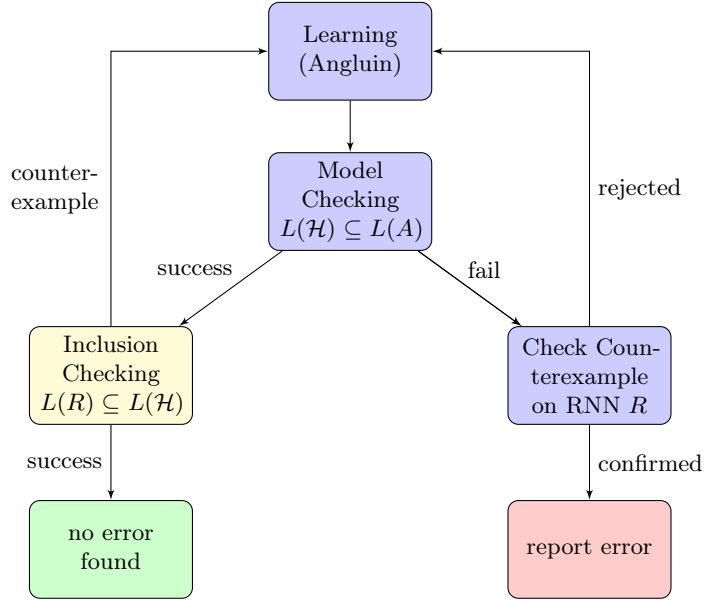
\begin{figure}[t]
	\centering
	\begin{tikzpicture}[node distance = 2cm, auto]
	
	    \node [block] (learning) {Learning (Angluin)};
	    \node [block, below of=learning] (mc) {Model Checking $L(\Hyp) \subseteq L(\A)$};
	    \node [varblock, below left=1cm and 1cm of mc] (equivalence) {Inclusion Checking $L(\RNN) \subseteq L(\Hyp)$};
	    \node [block, below right=1cm and 1cm of mc] (counter) {Check Counterexample on RNN $\RNN$};
	    \node [redblock, below=1cm of counter] (false) {report error};
	    \node [greenblock, below=1cm of equivalence] (true) {no error found};
	
	    \path [line] (learning) -- (mc);
	    \path [line] (mc) -- (counter);
	    \path [line] (mc) -- node [near start, left, xshift=-2mm] {success} (equivalence);
	    \path [line] (mc) -- node [near start, right, xshift=2mm] {fail} (counter);
	
	    \path [line] (counter) -- node {confirmed}(false);
	    \path [line] (equivalence) -- node[left] {success}(true);
	    \path [line] (counter)|- node [near start, right] {rejected} (learning);
	    \path [line] (equivalence)|- node [near start, left] {\begin{tabular}{c}counter-\\example\end{tabular}} (learning);
	    
	\end{tikzpicture}
	
	\caption{Property-directed verification of RNNs} \label{fig:work-flow}
\end{figure}

\paragraph{Experimental Evaluation.}
In fact, we compare our proposal with two natural alternatives.
The first method is (pure) statistical model checking,
which evaluates the network on the basis of random samples. 
The second method is to first extract from the given RNN
a finite-state model that is then verified using a standard
model-checking procedure. 
Note that our method can be seen as a combination of both where,
in addition, extraction and verification are intertwined.

Our experimental results show that misclassified inputs are found
faster than in the cases of pure statistical sampling and
model checking on an extracted finite automaton.
A further key advantage of our approach is that,
unlike in statistical model checking,
we often find a \emph{family} of counterexamples,
in terms of loops in the hypothesis automaton, which
testify conceptual problems of the given RNN.

\paragraph{Outline.}
The next subsection describes further related work.
In Section~\ref{sec:preliminaries}, we recall basic notions such as RNNs
and finite automata. Section~\ref{sec:verification} describes two basic algorithms for the verification of RNNs,
before we present property-directed verification in Section~\ref{sec:prop-directed}.
The experimental evaluation and a thorough discussion can be found in Section~\ref{sec:evaluation}.

\subsection*{Related Work}

Our approach to the verification of RNNs
relies on black-box checking, which has been designed
as a combination of model checking and testing
in order to verify finite-state systems \cite{PeledVY02}.

Mayr and Yovine describe an adaptation of the PAC variant of Angluin's L* algorithm
that can be applied to neural networks \cite{MayrY18}.
As L* is not guaranteed to terminate when facing non-regular
languages, the authors impose a bound on the number of states of the
hypotheses and on the length of the words for membership queries. 

Weiss et al.\ introduce a different technique to extract
finite automata from RNNs \cite{WeissGY18a}. It also
relies on Angluin's L*
but, moreover, uses an orthogonal abstraction of the given
RNN to perform equivalence checks between them. 
Since both are hypotheses for the RNN ground truth, consequently, as acknowledged by the authors, 
the convergence of the two hypotheses cannot  guarantee their equivalence to the underlying network.  
Like in \cite{MayrY18},
this approach leads to finite-state models
that can be used to analyze the underlying network.
Though formal verification was not the principal goal,
this approach turns out to be useful for finding
misclassified examples.

Ayache et al.\ study how to extract weighted finite automata from RNNs without 
applying equivalence queries \cite{AyacheEG18}. To solve the same problem, the work \cite{OkudonoWSH20} 
employs an L*-like method that resorts to regression 
techniques to answer equivalence queries, where 
counterexample candidates are prioritized by exploiting the internal state space of the given RNN.

Other extraction methods go back to the 90s and combine
partitioning the state space with transition
sampling \cite{OmlinG96}.

Elboher et al.\ present a counter-example guided
verification framework whose workflow shares similarities with
our property-guided verification \cite{ElboherGK20}. However, their approach
addresses feed-forward neural networks (FFNNs) rather than RNNs
and, therefore, involves neither automata-learning algorithms
nor classical model checking. For recent progress in the area
of safety and robustness verification of deep neural networks,
see~\cite{Kwiatkowska:2019}.

The paper \cite{AkintundeKLP19} studies formal verification of
systems where an RNN-based agent interacts with a
linearly definable environment.
The verification procedure proceeds by a reduction to
FFNNs. It is complete and fully automatic. This is at the expense of
the expressive power of the specification language, which is restricted to properties that only depend on
bounded prefixes of the system's executions. In our approach, we do not restrict the
kind of regular property to verify.

The work \cite{JBK-arxiv2020} also reduces the verification of RNNs to FFNN verification. To do so, the authors calculate inductive invariants, thereby avoiding a blowup in the network size. The effectiveness of their approach is demonstrated on audio signal systems. 
Like in \cite{AkintundeKLP19}, a time interval is imposed in which a given property is verified.


\section{Preliminaries}\label{sec:preliminaries}

In this section, we provide definitions of basic concepts
such as languages, recurrent neural networks,
finite automata, and Angluin's L* algorithm.
Hereby, we assume familiarity with
basic notions from formal languages and automata theory.

\paragraph{Words and Languages.}
Let $\Sigma$ be an alphabet, i.e., a nonempty finite set,
whose elements are called \emph{letters}. A (finite) word $w$ over $\Sigma$
is a sequence $a_1 \ldots a_n$ of letters $a_i \in \Sigma$.
The length of $w$ is defined as $|w| = n$. The unique word
of length $0$ is called the \emph{empty word} and denoted by
$\emptyword$. We let $\Sigma^\ast$ refer to the set of all words over $\Sigma$.
Any set $L \subseteq \Sigma^\ast$ is called
a \emph{language} (over $\Sigma$).
For two languages $L_1,L_2 \subseteq \Sigma^\ast$,
we let $L_1 \setminus L_2 = \{w \in \Sigma^\ast \mid w \in L_1$ and $w \not\in L_2\}$.
The symmetric difference of $L_1$ and $L_2$
is defined as $L_1 \oplus L_2 = (L_1 \setminus L_2) \cup (L_2 \setminus L_1)$.

Let $X$ be any set and $f: X \times \Sigma \to X$.
We can extend $f$ to a mapping $\ext{f}: X \times \Sigma^\ast \to X$
by $\ext{f}(x,\emptyword) = x$ and
$\ext{f}(x,aw) = \ext{f}(f(x,a),w)$. In recurrent neural networks and
finite automata, which are defined below,
$X$ will be instantiated by a set of states, and
$f$ will play the role of a transition function.

In order to sample words over $\Sigma$, we assume
a probability distribution $(\letterprobp{a})_{a \in \Sigma}$ on $\Sigma$
(by default, we pick the uniform distribution) and a
``termination'' probability $\tprob \in [0,1]$. Together, they determine a
natural probability distribution on $\Sigma^\ast$ given, for $w = a_1 \ldots a_n \in \Sigma^\ast$, by
$P(w) = \letterprobp{a_1} \cdot \ldots \cdot \letterprobp{a_n} \cdot (1-\tprob)^n \cdot \ntprob$.
According to the geometric distribution, the expected length of a word
is $(1/\tprob) -1$, with a variance of $(1-\tprob)/\tprob^2$.

Let $\epsilon > 0$ be an error parameter and $L_1,L_2 \subseteq \Sigma^\ast$
be languages.
We call $L_1$ $\epsilon$-\emph{approximately correct} wrt. $L_2$
if $\Prob(L_1 \setminus L_2) = \sum_{w \in L_1 \setminus L_2} \Prob(w) < \epsilon$.
Note that this is not a symmetric relation, as our notion of correctness
relies on inclusion rather than symmetric difference.

\paragraph{Recurrent Neural Networks.}
Recurrent neural networks (RNNs)
are a generic term for artificial neural networks that process sequential data.
They are particularly suitable for classifying sequences of varying length,
which is essential in domains such as natural language processing (NLP) or
time-series prediction.

Formally, a \emph{recurrent neural network (RNN)} over
$\Sigma$ is given by a tuple
$\RNN = (\sdim,\transrnn,\initrnn,\accrnn)$. Here,
$\sdim \in \N$ is the dimension of the \emph{state space}
$\R^\sdim$, which contains the \emph{initial state}
$\initrnn \in \R^\sdim$. Moreover,
$\transrnn: \R^\sdim \times
\Sigma \to \R^\sdim$ is the transition function
describing the effect of applying an input letter in a given source
state. The mapping
$\accrnn: \R^\sdim \to \{0,1\}$ identifies rejecting (0) and accepting (1) states.
The language of $\RNN$ is then defined as
$L(\RNN) = \{w \in \Sigma^\ast \mid \accrnn(\ext{\transrnn}(\initrnn,w)) = 1\}$.
In other words, a word $w$ is contained in $L(\RNN)$ if,
starting from $\initrnn$, the state
that we obtain by applying
$\transrnn$ successively to the letters of $w$
is accepting.

There are several ways to effectively represent
the functions $\transrnn$ and $\accrnn$. Among the most popular architectures
are (simple) Elman RNNs, long short-term memory (LSTM) \cite{DBLP:journals/neco/HochreiterS97},
and GRUs \cite{ChoMGBBSB14}.
Their expressive power depends on the exact architecture, but
generally goes beyond the power of finite automata, i.e., the class
of regular languages.

\paragraph{Finite Automata.}
RNNs provide a very powerful learning framework.
However, once trained, their internal structure is difficult to understand
and to verify.
In contrast, while generally being less expressive, the simplicity
of finite automata provides a simple framework for modeling, specifying, and
verifying computer systems.
The main difference to RNNs is that their state space is doomed to be finite.

Formally, a \emph{deterministic finite automaton (DFA)} over $\Sigma$ is a tuple
$\A = (Q,\delta,q_0,F)$ where $Q$ is a finite set
of states, $q_0 \in Q$ is the initial state,
$F \subseteq Q$ is the set of final states, and
$\delta \colon Q \times \Sigma \to Q$ is the transition function.
The language of $\A$ is defined as
$L(\A) = \{w \in \Sigma^\ast \mid \ext{\delta}(q_0,w) \in F\}$.
A language $L \subseteq \Sigma^\ast$ is called \emph{regular} if
$L = L(\A)$ for some DFA $\A$.

We sometimes use RNNs and DFAs synonymous for their respective
languages. For example, we say that $\RNN$ is $\width$-approximately
correct wrt. $\A$ if $L(\RNN)$ is $\width$-approximately
correct wrt. $L(\A)$.

\paragraph{Temporal Logics and Regular Expressions.}

Linear-time temporal logic (LTL)
has recently attracted renowned interest in the AI community \cite{GiacomoV15}.
As opposed to finite automata, LTL has
the advantage of being modular and easy to use, because many specifications
involving natural-language constructs such as ``at some point in the future''
or ``always in the future'' have precise formal counterparts in LTL.
Similarly, regular expressions offer an intuitive way of specifying
regular languages.
In the following, we assume that a property given in terms
of an LTL formula or a regular expression has already been
compiled into a corresponding DFA $\A$.

\paragraph{Angluin's Algorithm.}
Angluin introduced $\Lstar$, a classical instance of a learning algorithm
in the presence of a minimally adequate teacher (MAT) \cite{Angluin87}.
Its goal is to come up with a DFA that recognizes a given regular language
$L \subseteq \Sigma^\ast$.
The crux is that, while $\Sigma$ is given, $L$ is a priori unknown and can
only be accessed through
\emph{membership queries (MQ)} and \emph{equivalence queries (EQ)}.

Given any regular language $L \subseteq \Sigma^\ast$,
Angluin's algorithm $\Lstar$
eventually outputs the unique minimal DFA $\Hyp$ such that $L(\Hyp) = L$.
We do not detail the algorithm here but only define the
interfaces that we need to embed $\Lstar$ into our framework.
$\Lstar$ may ask whether
\begin{description}\itemsep=0.5ex
	\item[(MQ)] $w \mathrel{\smash{\stackrel{?}{\in}}} L$ for an arbitrary word $w \in \Sigma^\ast$. Thus, the answer is either yes or no.
	\item[(EQ)] $L(\Hyp) \stackrel{\smash{?}}{=} L$ for an arbitrary DFA $\Hyp$. Again, the answer is  either yes or no. If the answer is no, one also gets a counterexample word from the symmetric difference $L(\Hyp) \oplus L$.
\end{description}
Essentially, $\Lstar$ asks MQs until it considers that it has a consistent data set to come up with a hypothesis DFA $\Hyp$, which then undergoes an EQ. If the latter succeeds, then the algorithm stops. Otherwise, the counterexample and possibly more membership queries are used to refine the hypothesis. The algorithm provides the following guarantee:
If MQs and EQs are answered according to a given regular language $L \subseteq \Sigma^\ast$,
then the algorithm eventually outputs,
after polynomially\footnote{in the index
of the right congruence associated with $L$ and in the size of the
longest counterexample provided by the oracle}
many steps, the (unique) minimal DFA $\Hyp$ such that $L(\Hyp) = L$.

Angluin also proposed a \emph{probably approximately correct (PAC)} variant of her algorithm,
where EQs are implemented in terms of random MQs, namely as statistical tests~\cite{Angluin87}. An EQ
is considered successful if a (large enough) number of words sampled according to
a given probability distribution on words agree on both $L(\Hyp)$ and $L$.
We get back to this kind of sampling further below.


\section{Verification Approaches}\label{sec:verification}

Before we present (in Section~\ref{sec:prop-directed}) our method
of verifying RNNs, we here describe two simple approaches.
The experiments will later compare all three algorithms
wrt.\ their performance.

\paragraph{Statistical Model Checking (SMC).}
The obvious approach for checking whether the RNN under test $\RNN$
satisfies a given specification $\PropAut$, i.e., to check whether
$L(\RNN) \subseteq L(\PropAut)$, is by a form of random
testing. The idea is to generate a finite test suite
$T \subset \Sigma^\ast$ and to check, for each $w \in T$, whether for
$w \in L(\RNN)$ also $w \in L(\PropAut)$ holds. If not, each such $w$
is a \emph{counterexample}. On the other hand, if none of the words turns out to be a counterexample,
the property holds on $\RNN$ with a certain error probability.
The algorithm is sketched as
Algorithm~\ref{algo:statistical_model_checking}.

Note that the test suite is sampled according to a probability distribution 
on $\Sigma^\ast$. Recall that our choice
depends on two parameters:
a probability distribution $(\letterprobp{a})_{a \in \Sigma}$ on $\Sigma$
(by default the uniform distribution) and a
``termination'' probability $\tprob \in [0,1]$.
Moreover, for $w = a_1 \ldots a_n \in \Sigma^\ast$, we set
$P(w) = \letterprobp{a_1} \cdot \ldots \cdot \letterprobp{a_n} \cdot (1-\tprob)^n \cdot \ntprob$.

\begin{algorithm}[t]
	\DontPrintSemicolon
    
	\KwIn{An RNN $\RNN$, a DFA $\PropAut$, and $\width,\confidence \in (0, 1)$}
	\BlankLine

	\For{$i = 1, \ldots, \log(2 / \width) / (2 \confidence^2)$}
	{

		$w \gets \text{sampleWord()}$\;

		\uIf{$w \in L(\RNN) \text{ and } w \not\in L(\PropAut)$}
		{
			\Return \colorcounter{``Counterexample $w$''}\;
		}
	}
	
	\Return \colorsat{``Property satisfied''}\;
    
	\caption{Statistical Model Checking of RNNs}
	\label{algo:statistical_model_checking}    
\end{algorithm}

\begin{theorem}[Correctness of Statistical Model Checking]
If Algorithm~\ref{algo:statistical_model_checking}, with $\width,\confidence \in (0,1)$, terminates with
``Counterexample $w$'', then $w$ is mistakenly classified by $\RNN$ as positive.
If it terminates with ``Property satisfied'',
then $\RNN$ is $\epsilon$-approximately correct wrt.\ $\A$ with probability at least $1-\confidence$.
\end{theorem}

\begin{proof}
If the algorithm terminates with ``Counterexample $w$'', we have $w \in L(\RNN) \setminus L(\A)$. Thus, $w$ is mistakenly classified.
Using the sampling described in Section~\ref{sec:preliminaries}, denote by $\hat p$ the probability to pick $w\in \Sigma^\ast$ such that $w \in L(\RNN) $ and $ w \not\in L(\PropAut)$.
Taking $n = \log(2 / \width) / (2 \confidence^2)$ random samples where $m$ of them are counter examples, by Hoeffding's inequality bound~\cite{Hoff63} we get that $P(\hat p\notin [\frac{m}{n}-\width,\frac{m}{n}+\width])<\confidence.$	
Therefore, if Algorithm~\ref{algo:statistical_model_checking} terminates without finding any counterexamples we get that $\RNN$ is $\epsilon$-approximately correct wrt.\ $\A$ with probability at least $1-\confidence$.
\end{proof}

While the approach works in principle, it has several drawbacks for its
practical application. The size of the test suite
may be quite huge and it may take a while both finding a counterexample or proving
correctness.

Moreover, the correctness result and the algorithm assume that
the words to be tested are chosen according to a random distribution
that somehow also has to take into account the RNN as well as the
property automaton.

It has been reported that the method does not work well in practice
\cite{WeissGY18a} and our experiments support these
findings.

\paragraph{Automaton Abstraction and Model Checking (AAMC).} 
As model checking is mainly working for finite-state systems, a
straightforward idea would be to
\begin{enumerate*}[label={(\alph*)}]
	\item \emph{approximate} the RNN $\RNN$ by a finite-state automaton $\A_\RNN$ such that $L(\RNN) \approx
L(A_\RNN)$ and 
	\item to check whether $L(\A_\RNN) \subseteq L(\PropAut)$ using model checking.
\end{enumerate*}
The algorithmic schema is depicted in Algorithm~\ref{algo:automaton-abstraction-and-model-checking}.

Here, we can instantiate $\textup{Approximation}()$ by
the DFA-extraction algorithms from \cite{WeissGY18a} or \cite{MayrY18}.
In fact, for approximating an RNN by a finite-state system, several approaches
have been studied in the literature, which can be, roughly, divided
into two approaches:
\begin{enumerate*}[label={(\alph*)}]
	\item \emph{abstraction} and 
	\item \emph{automata learning}.
\end{enumerate*}
In the first approach, the state space of the RNN is mapped to
equivalence classes according to certain predicates. The second
approach uses automata-learning techniques such as Angluin's
L$^\ast$. The approach \cite{WeissGY18a} is an intertwined version
combining both ideas.

Therefore, there are different instances of AAMC, varying in the approximation
approach. Note that, for verification as language inclusion, as
considered here, it actually suffices to learn an over-approximation $\A_\RNN$
such that $L(\RNN) \subsetsim L(\A_\RNN)$.

\begin{algorithm}[t]
	\DontPrintSemicolon
	
	\KwIn{An RNN $\RNN$ and a DFA $\PropAut$}
	\BlankLine
	
	$\A_\RNN \gets \text{Approximation}(\RNN)$\label{algo:aamc_aprox}\;
	\uIf{$\exists w \in L(\A_\RNN) \setminus L(\PropAut)$}
	{
        	\Return \colorcounter{``Counterexample $w$''}\;
	}
	\lElse
	{
        	\Return \colorsat{``Property satisfied''}
	}
	
	\caption{Automaton-Abstraction-and-MC}
	\label{algo:automaton-abstraction-and-model-checking}
\end{algorithm}

While the approach seems promising at first hand, its correctness has
two glitches. First, the result ``Property satisfied'' depends on the
quality of the approximation. Second, any returned counterexample $w$
may be \emph{spurious}: $w$ is a counterexample with respect to
$\A_\RNN$ satisfying $\PropAut$ but may not be a counterexample for
$\RNN$ satisfying $\PropAut$.

If $w \in L(\RNN)$, then it is indeed a counterexample, but if not, it
is spurious -- an indication that the approximation needs to be refined. If the
automaton is obtained using abstraction techniques (such as 
predicate abstraction) that guarantee over-approximations,
well-known principles like CEGAR \cite{ClarkeGJLV00}
may be used to refine it. In the automata-learning setting,
$w$ may be used as a counterexample for the
learning algorithm to improve the approximation. 

Repeating the latter idea suggests an interplay between automata learning and
verification -- and this is the idea that we follow in this
paper. However, rather than starting from some approximation
with a certain quality that is later refined according to the RNN
and the property, we perform a direct, \emph{property-directed} approach.


\section{Property-Directed Verification of RNNs}
\label{sec:prop-directed}

We now provide our algorithm for property-directed verification (PDV).
The underlying idea is to replace the EQ in Angluin's $\Lstar$ algorithm with a combination of classical model checking and statistical model checking, which are used as an alternative to EQs.
This approach, which we call \emph{property-directed verification of RNNs}, is outlined as Algorithm~\ref{algo:Lstar} and works as follows.

After initialization of $\Lstar$ and the corresponding data structure,
$\Lstar$ automatically generates and asks MQs to the given RNN $\RNN$
until it comes up with a first hypothesis DFA $\Hyp$ (Line~\ref{algo:Lstar:A}).
In particular, the language $L(\Hyp)$ is consistent with the MQs
asked so far.

At an early stage of the algorithm, $\Hyp$ is generally small.
However, it already shares some characteristics with $\RNN$.
So it is worth checking, using standard automata algorithms,
whether there is no mismatch yet
between $\Hyp$ and $\A$, i.e., whether $L(\Hyp) \subseteq L(\A)$
holds (Line~\ref{algo:Lstar:B}).
Because otherwise (Line~\ref{algo:Lstar:H}), a counterexample word
$w \in L(\Hyp) \setminus L(\A)$ is already a candidate for being
a misclassified input for $\RNN$.
If indeed $w \in L(\RNN)$, $w$ is mistakenly considered positive
by $\RNN$ so that $\RNN$ violates the specification $\A$.
The algorithm then outputs ``Counterexample $w$'' (Line~\ref{algo:Lstar:K}).
If, on the other hand, $\RNN$ happens to agree with $\A$ on a
negative classification of $w$, then there is a mismatch between $\RNN$ and
the hypothesis $\Hyp$ (Line~\ref{algo:Lstar:L}). In that case, $w$ is fed back to
$\Lstar$ to refine $\Hyp$.

Now, let us consider the case that $L(\Hyp) \subseteq L(\A)$
holds (Line~\ref{algo:Lstar:C}). If, in addition, we can establish $L(\RNN) \subseteq L(\Hyp)$,
we conclude that $L(\RNN) \subseteq L(\A)$ and output
``Property satisfied'' (Line~\ref{algo:Lstar:F}). This inclusion test (Line~\ref{algo:Lstar:D}) relies
on statistical model checking using given parameters $\epsilon,\gamma >0$
(cf.\ Algorithm~\ref{algo:statistical_model_checking}).
If the test passes, we have some statistical guarantee of correctness of $\RNN$
(cf. Theorem~\ref{thm:main}). Otherwise, we obtain a word
$w \in L(\RNN) \setminus L(\Hyp)$ witnessing a discrepancy between
$\RNN$ and $\Hyp$ that will be exploited to refine $\Hyp$ (Line~\ref{algo:Lstar:G}).

\begin{algorithm}[t]
	\DontPrintSemicolon
	
	\KwIn{An RNN $\RNN$, a DFA $\PropAut$, and $\varepsilon, \gamma \in (0, 1)$}
	\BlankLine
	
	Initialize $\Lstar$\;
	\BlankLine
	
	\While{true}
	{
	
		$\Hyp \gets \text{ hypothesis provided by } \Lstar$ \label{algo:Lstar:A}\;
		Check $L(\Hyp) \subseteq L(\PropAut)$ \label{algo:Lstar:B}\;
	
		\uIf{$L(\Hyp) \subseteq L(\PropAut)$}
		{
			\label{algo:Lstar:C}
			 
			Check $L(\RNN) \subseteq L(\Hyp)$ using Algorithm~\ref{algo:statistical_model_checking} \label{algo:Lstar:D}\;
			
			\uIf{$L(\RNN) \subseteq L(\Hyp)$}
			{
				\label{algo:Lstar:E}
			
				\Return \colorsat{``Property satisfied''}  \label{algo:Lstar:F}\;
			}
			\lElse
			{
				Feed counterexample back to $\Lstar$  \label{algo:Lstar:G}
			}
			
		}
		\Else
		{
			\label{algo:Lstar:H}

			Let $w \in L(\Hyp) \setminus L(\PropAut)$ \label{algo:Lstar:I}\;
		
			\uIf{$w \in L(\RNN)$}
			{
				\label{algo:Lstar:J}
			
				\Return \colorcounter{``Counterexample $w$''} \label{algo:Lstar:K}\;
			}
			\lElse
			{
				Feed $w$ back to $\Lstar$ \label{algo:Lstar:L}
			}
		}
	}

	\caption{Property-directed verification of RNNs}
	\label{algo:Lstar}
\end{algorithm}

Overall, in the event that the algorithm terminates, we get the following guarantees:

\begin{theorem}\label{thm:main}
Suppose Algorithm~\ref{algo:Lstar} terminates, using SMC for inclusion checking with
parameters $\epsilon$ and $\gamma$.
If it outputs ``Counterexample $w$'', then $w$ is mistakenly classified by $\RNN$ as positive.
If it outputs ``Property satisfied'', then
$\RNN$ is $\epsilon$-approximately correct wrt.\ $\A$
with probability at least $1-\gamma$.
\end{theorem}

\begin{proof}
Suppose the algorithm outputs ``Counterexample $w$'' in Line~\ref{algo:Lstar:K}.
Due to Lines~\ref{algo:Lstar:I} and \ref{algo:Lstar:J}, we have $w \in L(\RNN) \setminus L(\A)$.
Thus, $w$ is a counterexample.

Suppose the algorithm outputs ``Property satisfied'' in Line~\ref{algo:Lstar:F}. 
By Lines~\ref{algo:Lstar:D} and \ref{algo:Lstar:E},
$\RNN$ is $\epsilon$-approximately correct wrt.\ $\Hyp$
with probability at least $1-\gamma$.
That is, $P(L(\RNN) \setminus L(\Hyp)) < \epsilon$ with high probability.
Moreover, by Line~\ref{algo:Lstar:B}, $L(\Hyp) \subseteq L(\A)$.
This implies that $L(\RNN) \setminus L(\A) \subseteq L(\RNN) \setminus L(\Hyp)$
and, therefore, 
$P(L(\RNN) \setminus L(\A)) \le P(L(\RNN) \setminus L(\Hyp))$.
We deduce that
$\RNN$ is $\epsilon$-approximately correct wrt.\ $\A$
with probability at least $1-\gamma$.
\end{proof}

Although we cannot hope that Algorithm~\ref{algo:Lstar} will always terminate, we demonstrate empirically that Algorithm~\ref{algo:Lstar} is an effective way for the verification of RNNs.


\section{Experimental Evaluation}
\label{sec:evaluation}

We now present an experimental evaluational of the three model-checking algorithms SMC, AAMC, and PDV, and provide a comparison of their performance on LSTM networks~\cite{DBLP:journals/neco/HochreiterS97} (a variant of RNNs using LSTM units). 
The algorithms have been implemented\footnote{publicly available at \url{https://github.com/LeaRNNify/Property-directed-verification}} in Python 3.6 using PyTorch 19.09 and Numpy library.
In the implementation, the approximation used for AAMC (Line \ref{algo:aamc_aprox}) is Angluin's PAC version described in \cite{Angluin87} (which was also implemented in~\cite{MayrY18}).
All of the experiments were run on NVIDIA DGX-2 with an Ubuntu OS.

\paragraph{Optimization For Equivalence Queries.} 
In \cite{MayrY18}, the authors implement AAMC but with an optimization that was originally shown in \cite{Angluin87}. This optimization concerns the number of samples required for checking the equivalence between the hypothesis and the taught language. This number depends on $\epsilon, \gamma$ as well as the number for previous equivalence queries $n$ and is calculated by 
$ \frac{1}{\epsilon} \left( \log\frac{1}{\gamma}+\log(2)(n+1) \right)$.
We adopt this optimization in AAMC and PDV as well (Algorithm \ref{algo:automaton-abstraction-and-model-checking} in Line~\ref{algo:aamc_aprox} and Algorithm~\ref{algo:Lstar} in Line~\ref{algo:Lstar:D}).

\paragraph{Synthetic Benchmarks.}
In order to compare the algorithms, we implemented the following procedure which generates a random DFA $\randA$, an RNN $\RNN$ that learned $L(\randA)$, and a finite set of specification DFAs $\specA_1, \specA_2,\ldots$ as follows:
\begin{enumerate}
	\item Choose a random DFA $\randA = (Q,\delta,q_0,F)$, with $|Q| \leq 30$, over an alphabet $\Sigma$ with $|\Sigma| = 5$.
	\item Randomly sample words from $\Sigma^\ast$ as described in Section~\ref{sec:preliminaries} in order to create a training set and a test set.
	\item Train an RNN $\RNN$ with hidden dimension $20|Q|$ and $1 + |Q|/10$ layers. If the accuracy of $\RNN$ on the training set is larger than $95\% $, continue. Otherwise, restart the procedure.
	\item Choose randomly up to five sets $F_i = Q\setminus F$ to define specification DFAs $\specA_i=(Q,\delta,q_0,F\cup F_i)$. 
\end{enumerate}
Using this procedure, we created 30 DFAs\slash RNNs and 138 specifications.

\paragraph{Experimental Results.}
Given an RNN $R$ and a specification DFA $\specA$, we checked whether $R$ satisfies $\specA$ using Algorithms~1--3, i.e., SMC, AAMC, and PDV, with $\epsilon, \gamma= 5\cdot10^{-4}$. 

Table~\ref{tbl:summary_benchmark} summarizes the executions of the three algorithms on our 138 random instances. The columns of the table are as follows:
\begin{enumerate*}[label={(\roman*)}]
	\item the average time was counted in seconds and all the algorithms were timed out after 10 minutes,
	\item \emph{Avg len} is the average length of the found counter examples (if one was found),
	\item \emph{\#~Mistakes} is the number of random instances for which a mistake was found, and
	\item \emph{Avg MQs} is the average number of membership queries asked to the RNN.
\end{enumerate*}

\begin{table}[t]
	\caption{Experimental results\label{tbl:summary_benchmark}} \label{tbl:summary_benchmark}

	\centering

	\begin{tabular}{c*{4}{r}}
		\toprule
		Type & ~~~\emph{Avg time} (s) & ~~~\emph{Avg len} & ~~~\emph{\# Mistakes} & ~~~~~\emph{Avg MQs} \\
		\midrule
		SMC& 92   & 111 & \textbf{122} & 286063 \\
		AAMC& 444 & \textbf{7} & 30 &3701916 \\
		PDV& \textbf{21}  & 11 & 109  &\textbf{28318} \\
		\bottomrule
	\end{tabular}

\end{table}

Note that not only is PDV faster and finds more errors than AAMC, the average number of states of the final DFA is also much smaller: {\bf26} states with PDV and {\bf319} with AAMC.

Comparing PDV to SMC, it is 4.5 times faster and it asked 10 times less MQs from the RNN, even though it found a little less mistakes. The time difference can become even more apparent if PDV was implemented in something other than Python. Another observation is that the lengths of the counterexamples are much smaller in PDV.

\paragraph{Faulty Flows.}
One of the advantages of extracting DFAs in order to detect mistakes in a given RNN is the possibility to find not only one mistake but a ``faulty flow''. 
For example, Figure~\ref{fig:extracted_faulty_flow} shows one DFA extracted with PDV, based on which we found a mistake in the corresponding RNN. The counter example we found was $ (a,b,c,e,e) $. One can see that the word $ (a,b,c,e) $ is a loop in the DFA. Hence we can suspect that this could be a ``faulty flow''. Checking the words $w_n = (a,b,c,e)^n(e) $ for $n\in[1..100]$, we observed that for any $n\in[1..100] $ the word $w_n$ was in the RNN language but not in the specification.

\begin{figure}
	\centering

	\begin{tikzpicture}
		\node[state] (0) at (0, 0) {$0$};
		\node[state] (1) at (2.25, 0) {$1$};
		\node[state] (2) at (-2.25, -2.25) {$2$};
		\node[state] (3) at (0, -2.25) {$3$};
		\node[state] (4) at (2.25, -2.25) {$4$};
		
		\draw[<-, shorten <=1pt] (0.west) -- +(-.3, 0);
		
		\begin{scope}[every node/.append style={font=\small}]
			\path[->] (0) edge[bend left=20,red,thick] node {$e$} (1) edge[thick,blue] node {$a, c$} (3) edge node[sloped] {$b, d$} (4);
			\path[->] (1) edge node {$a$} (0) edge node {$b, c, d, e$} (4);
			\path[->] (2) edge[blue,thick] node {$e$} (0) edge[blue,loop above,thick] node {$b, c, d$} () edge node {$a$} (3);
			\path[->] (3) edge[blue,thick,bend left=20] node {$b$} (2) edge node {$a, c, d, e$} (4);
			\path[->] (4) edge[bend left=20] node {$a, c$} (3) edge[loop right] node {$b, d, e$} ();
		\end{scope}
	\end{tikzpicture}
	\caption{Faulty Flow in DFA extracted through PDV} \label{fig:extracted_faulty_flow}
\end{figure}
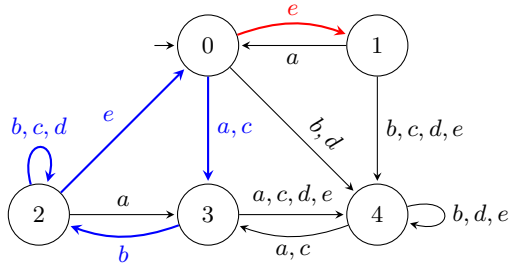

To automate the reasoning above, we did the following:
Given an RNN $\RNN$, a specification $\specA$, the extracted DFA $\Hyp$, and the counter example $w$:
\begin{enumerate}
	\item Build the cross product DFA: $A_\times = \specA \times \Hyp$.\\
	\item For every prefix $w_1$ of the  counter example $w = w_1w_2$, denote by $s_{w_1}$ the state to which the prefix $w_1$ leads in $ A_\times $. For any loop $ \ell$ starting from $s_{w_1}$, check if $w^f_n = w_1\ell^n w_2 $ is a counter example for $n\in[1...100]$. \\
	\item If $w^f_n$ is a counter example for more than 20 times, declare ``found a faulty flow''.
\end{enumerate}
Using this procedure, we managed to find faulty flows in 81\slash 109 of the counterexamples that were found by PDV.

\paragraph{RNNs Identifying Contact Sequences.} 
Contact tracing~\cite{2002principles} has proven to be increasingly effective in curbing the spread of infectious diseases.
In particular, analyzing contact sequences---sequences of individuals who have been in close contact in a certain order---can be crucial in identifying individuals who might be at risk during an epidemic.    
We, thus, look at RNNs which can potentially aid contact tracing by identifying possible contact sequences.   
However, in order to deploy such RNNs in practice, one would require them to be verified adequately. One does not want to alert individuals unnecessarily even if they are safe or overlook individuals who could be at risk.

In a real world setting, one would obtain contact sequences from contact-tracing information available from, for instance, contact-tracing apps. 
However, such data is often difficult to procure due to privacy issues. 
Thus, in order to mimic a real life scenario, we use data available from \url{www.sociopatterns.org} which contains information about interaction of humans in public places (hospitals, schools, etc.) presented as temporal networks.

Formally, a \emph{temporal network} $\temporalnet=(\vertices,\edges)$~\cite{DBLP:reference/snam/Holme14} is a graph structure consisting of a set of vertices $\vertices$ and a set of labeled edges $\edges$, where the labels represent the timestamp during which the edge was active.
Figure~\ref{fig:temporal-network} is a simple temporal network, which can perceived as contact graph of 4 workers in an office where edge labels represent the time of meeting between them. 
A \emph{time-respecting path} $\pi\in \vertices^\ast$---a sequence of vertices such that there exists a sequence of edges with increasing time labels---depicts a contact sequence in such a network. In the above example, $C$, $D$, $A$, $B$ is a time-respecting path while $A$, $B$, $C$, $D$ is not.

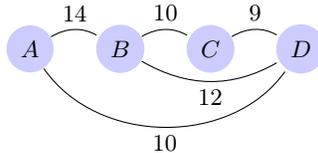
\begin{figure}
	\centering
	\begin{tikzpicture}[scale=.8,auto=center]
	\node [circle, fill = blue!20] (A) at (0,0)   {$A$};
	\node [circle, fill = blue!20] (B) at (1.5,0)   {$B$};
	\node [circle, fill = blue!20] (C) at (3,0)   {$C$};
	\node [circle, fill = blue!20] (D) at (4.5,0)   {$D$}; 
	\path [bend left] (A) edge node {\small{14}} (B);
	\path [bend left] (B) edge node {\small{10}} (C);
	\path [bend left] (C) edge node {\small{9}} (D);
	\path [bend right=55] (A) edge node[below] {\small{10}} (D);
	\path [bend right] (B) edge node[below] {\small{12}} (D);
	\end{tikzpicture}
	\caption{Temporal network representing depicting contact between 4 individuals}\label{fig:temporal-network}
\end{figure}

For our experiment, given a temporal network $\temporalnet$, we generate an RNN $R$ recognizing contact sequences as follows:  

\begin{enumerate}
	\item We create training and test data for the RNN by generating: valid time-respecting paths (of lengths between 5 and 15) using labeled edges from $\temporalnet$; and invalid time-respecting paths, by considering a valid path and randomly introducing breaks in the path. Number of time-respecting paths in the training set is twice the size of the number of labeled edges in $\temporalnet$, while the test set is one-fifth the size of the training set.
	\item We train RNN $R$ with hidden dimension $|V|$ (minimum 100) and $\floor{{2+{|V|}/100}}$ layers on the generated training data. We consider only those RNNs that could be trained within 5 hours with high accuracy (average 99\%) on the test data.
	\item As for the specification, we use a DFA which accepts all possible paths (disregarding the time labels) in the network. Using such a specification would allow us to check whether the RNN learned unwanted edges between vertices. 
\end{enumerate}
Using this process, from the seven temporal networks, we generate seven RNNs and seven specification DFAs.   
We ran the three algorithms, i.e., SMC, PDV, and AAMC, on the generated RNNs, using the same parameters as used for the random instances. We note the length of counterexample, the extracted DFA size (only for PDV and AAMC) and the running time of the algorithms (see Table~\ref{tbl:practical-example}).

\begin{table}
	\centering
	\begin{tabular}{clllr}
		\toprule
		& & \emph{Len. of}  & \emph{Extracted} &  \\
		\emph{Case} & \emph{Alg.} & \emph{counterexample}  & \emph{DFA size} & \emph{Time} (s) \\
		\midrule
		Across&SMC& 3 &  & 0.3 \\
		Kenyan&AAMC& 2 & 328 & 624.76  \\
		Household&PDV& 2 & 2 & 0.22 \\
		\midrule
		&SMC& 2 &  & 0.23 \\
		Workplace&AAMC& 2 & 111 & 604.99  \\
		&PDV& 2 & 2 & 0.77 \\
		\midrule
		&SMC& 5 & & 0.33 \\
		Highschool&AAMC& 2 & 91 & 627.30 \\
		2011&PDV& 2 & 2 & 0.19 \\
		\midrule
		&SMC& 7 & & 0.24 \\
		Hospital&AAMC& 2 & 36 & 614.76 \\
		&PDV& 2 & 2 & 0.006 \\
		\midrule
		Within&SMC& 2 & & 0.28 \\
		Kenyan&AAMC& 2 & 178 & 620.30 \\
		Household&PDV& 2 & 2 & 0.27 \\
		\midrule
		&SMC& 71 & & 1.51 \\
		Conference&AAMC& 2 & 38 & 876.19 \\
		&PDV& 2 & 2 & 0.33 \\
		\midrule
		&SMC& 3 & & 0.48 \\
		Workplace&AAMC& 2 & 87 & 621.44 \\
		2015&PDV& 2 & 2 & 1.11 \\
		\bottomrule\\
	\end{tabular}	
	\caption{Results of model-checking algorithm on RNN identifying contact sequences\label{tbl:practical-example}}
\end{table}

We make three main observations from Table~\ref{tbl:practical-example}. First, the counterexamples obtained by PDV and AAMC (avg. length 2), are much more succinct than those by SMC (avg. length 13.1). Small counterexamples help in identifying the underlying error in the RNN, while, long and random counterexamples provide much less insight. For example, from the counterexamples obtained from PDV and AAMC, we learned that the RNN overlooked certain edges or identified wrong edges. This result highlights the demerit of SMC, which has also been observed by~\cite{WeissGY18a}.    
Second, the running time of SMC and PDV (avg. 0.48 secs and 0.41 secs) is comparable while that of AAMC is prohibitively large (avg. 655.68 secs), indicating that model checking on small and rough abstractions of the RNN produces superior results. 
Third, the extracted DFA size, in case of AAMC (avg. size 124.14), is always larger compared to PDV (avg. size 2), indicating that RNNs are quite difficult to be approximated by small DFAs and this delays the model-checking process as well.  

Overall, our experiments confirm that PDV produces succinct counterexamples reasonably fast.


\section{Conclusion}

We proposed property-directed verification (PDV) as a new verification method for formally verifying RNNs with respect to regular specifications. To this end, we combined classical model checking with Angluin's L* algorithm, using statistical model checking for answering equivalence queries. PDV is (fully) correct when finding counterexamples (to the specification) and correct up to a given error probability when confirming correctness. While plain statistical model checking (SMC) and abstractions with model checking (AAMC) have similar guarantees, our evaluation suggests that PDV finds shorter counterexamples faster, which can even be generalized to identify faulty flows.

It is straightforward to extend our ideas to the setting of Moore/Mealy machines supporting the setting of richer classes of RNN classifiers, but this is left as part of future work.
Another future work is to investigate the applicability of our approach for RNNs representing more expressive languages, such as context-free ones. 
Note that the current approach is noise-free. It is worth exploring a noisy setting by studying 
how to minimize the noise introduced during RNN training. 
Finally, we plan to extend the PDV algorithm for the formal verification of RNN-based agent environment systems, and to compare it with the existing results for such closed loop systems controlled by RNNs.

\section*{Acknowledgments}
This work was partly supported by the PHC PROCOPE 2020 project \emph{LeaRNNify} (number 44707TK), funded by the German Academic Exchange Service (Deutscher Akademischer Austauschdienst) and Campus France.

\bibliographystyle{splncs04}
\bibliography{bib}

\end{document}